\documentclass[letterpaper, 10 pt, conference]{ieeeconf}  % Comment this line out if you need a4paper
\IEEEoverridecommandlockouts                              % This command is only needed if 
\usepackage{xcolor}
\usepackage{amssymb}
\usepackage{algorithm} 
\usepackage{algpseudocode}
\usepackage{graphicx}
\usepackage{amsmath}
\usepackage{mathtools, nccmath}
\DeclarePairedDelimiter{\nint}\lfloor\rceil

\usepackage{gensymb}
\usepackage{float}
\usepackage{multirow}
\usepackage{makecell}
\usepackage{mathtools}
\usepackage{hyperref}
\hypersetup{
    colorlinks=true,
    linkcolor=blue,
    filecolor=magenta,      
    urlcolor=blue,
    pdftitle={Overleaf Example},
    pdfpagemode=FullScreen,
    }

\usepackage{tabularx}
    \newcolumntype{L}{>{\raggedright\arraybackslash}X}
\usepackage[utf8]{inputenc}
\usepackage[english]{babel}
\usepackage{amsthm}

\usepackage{subcaption}

\newcommand{\no}{\noindent}

\newtheorem{lemma}{Lemma}[section]

\title{\LARGE \bf
TerraPN: Unstructured Terrain Navigation using Online Self-Supervised Learning 
}

\author{Adarsh Jagan Sathyamoorthy, Kasun Weerakoon, Tianrui Guan, Jing Liang and Dinesh Manocha \\% <-this % stops a space
\thanks{This research was supported by Army Cooperative Agreement W911NF2120076.}% <-this % stops a space
\small{Supplemental version including Tech Report, and Video at \url{http://gamma.umd.edu/terrapn/}}}

\begin{document}

\maketitle
\thispagestyle{empty}
\pagestyle{empty}

%%%%%%%%%%%%%%%%%%%%%%%%%%%%%%%%%%%%%%%%%%%%%%%%%%%%%%%%%%%%%%%%%%%%%%%%%%%%%%%%
\begin{abstract}
We present TerraPN, a novel method that learns the surface properties (traction, bumpiness, deformability, etc.) of complex outdoor terrains directly from robot-terrain interactions through self-supervised learning, and uses it for autonomous robot navigation. Our method uses RGB images of terrain surfaces and the robot's velocities as inputs, and the IMU vibrations and odometry errors experienced by the robot as labels for self-supervision. Our method computes a surface cost map that differentiates smooth, high-traction surfaces (low navigation costs) from bumpy, slippery, deformable surfaces (high navigation costs). We compute the cost map by non-uniformly sampling patches from the input RGB image by detecting boundaries between surfaces resulting in low inference times (47.27\% lower) compared to uniform sampling and existing segmentation methods. We present a novel navigation algorithm that accounts for a surface's cost, computes cost-based acceleration limits for the robot, and dynamically feasible, collision-free trajectories. TerraPN's surface cost prediction can be trained in $\sim 25$ minutes for five different surfaces, compared to several hours for previous learning-based segmentation methods. In terms of navigation, our method outperforms previous works in terms of success rates (up to 35.84\% higher), vibration cost of the trajectories (up to 21.52\% lower), and slowing the robot on bumpy, deformable surfaces (up to 46.76\% slower) in different scenarios.

% Our method learns correlations between the visual properties of different surfaces (color, texture, etc.) obtained from input RGB images with attributes such as traction, bumpiness, and deformability by measuring IMU vibrations, and robot's odometry errors by traversing on them with different velocities.  

% Original Abstract
% We present TerraPN, a novel method to learn the surface characteristics (texture, bumpiness, deformability, etc.) of complex outdoor terrains for autonomous robot navigation. Our method predicts navigability cost maps for different surfaces using patches of RGB images, odometry, and IMU data. Our method dynamically varies the resolution of the output cost map based on the scene to improve its computational efficiency. We present a novel extension to the Dynamic-Window Approach (DWA-O) to account for a surface's navigability cost while computing robot trajectories. DWA-O also dynamically modulates the robot's acceleration limits based on the variation in the robot-terrain interactions. In terms of perception, our method learns to predict navigability costs in $\sim 20$ minutes for five different surfaces, compared to Z hours for previous scene segmentation methods, and leads to a X$\%$ decrease in inference time. In terms of navigation, our method outperforms previous works in terms of vibration costs and generates robot velocities suitable for different surfaces.
\end{abstract}

%%%%%%%%%%%%%%%%%%%%%%%%%%%%%%%%%%%%%%%%%%%%%%%%%%%%%%%%%%%%%%%%%%%%
\section{Introduction}
Autonomous robots are currently being used for a variety of outdoor applications such as food/grocery delivery, agriculture, surveillance, planetary exploration, etc. designed for these applications need to account for a terrain's geometric properties such as slope or elevation changes, as well as its surface characteristics such as texture, bumpiness (level of undulations), softness/deformability, etc. to compute smooth and efficient robot trajectories \cite{weerakoon2021terp,sriram-siva-1,kahn2020badgr}. 

This is because, apart from a terrain's slope and elevation, its surface properties (texture, bumpiness, and deformability) govern its navigability for a robot. For instance, a surface's texture determines the traction experienced by the robot, its bumpiness determines the vibrations experienced, and deformability \cite{where-should-i-walk,rover-cmu} determines whether a robot could get stuck or experience wheel slips (e.g. in mud). Other factors that affect navigability over a terrain involve the robot's properties such as dynamics, inertia, physical dimensions, velocity limits, etc. 

Therefore, a key issue for performing smooth robot navigation on different terrains is perceiving and characterizing the surface properties. To this end, many works in computer vision, specifically semantic segmentation \cite{guan2021ganav,ttm,geo-visual} based on supervised learning, have demonstrated good terrain classification capabilities on RGB images. However, they rely on large image datasets with different classes of terrains annotated by humans. Such datasets do not account for a robot's properties and might misclassify a traversable terrain for a robot as non-traversable (or vice-versa). In addition, their classification outputs must be converted into quantities that measure a terrain's degree of navigability (costs) to be used for planning and navigation \cite{semantic-mapping,trav-analysis-terrain-mapping}. 
% Typically, such methods also demand substantial computational capabilities. 

% Such methods are trained using pixel-wise hand labeled image datasets and a predefined number of classes of terrains (e.g. smooth, rough, bumpy, forbidden, etc). However, these methods require carefully, pain-stakingly hand labeled datasets. Since they depend on human annotators for training, they do not account for a robot's dynamics and might misclassify a traversable terrain for a robot as non-traversable. In addition, such methods primarily classify terrains and do not compute quantitative metrics that can be directly used for planning and navigation. 

\begin{figure}[t]
      \centering
      \includegraphics[height=5.5cm, width=\columnwidth]{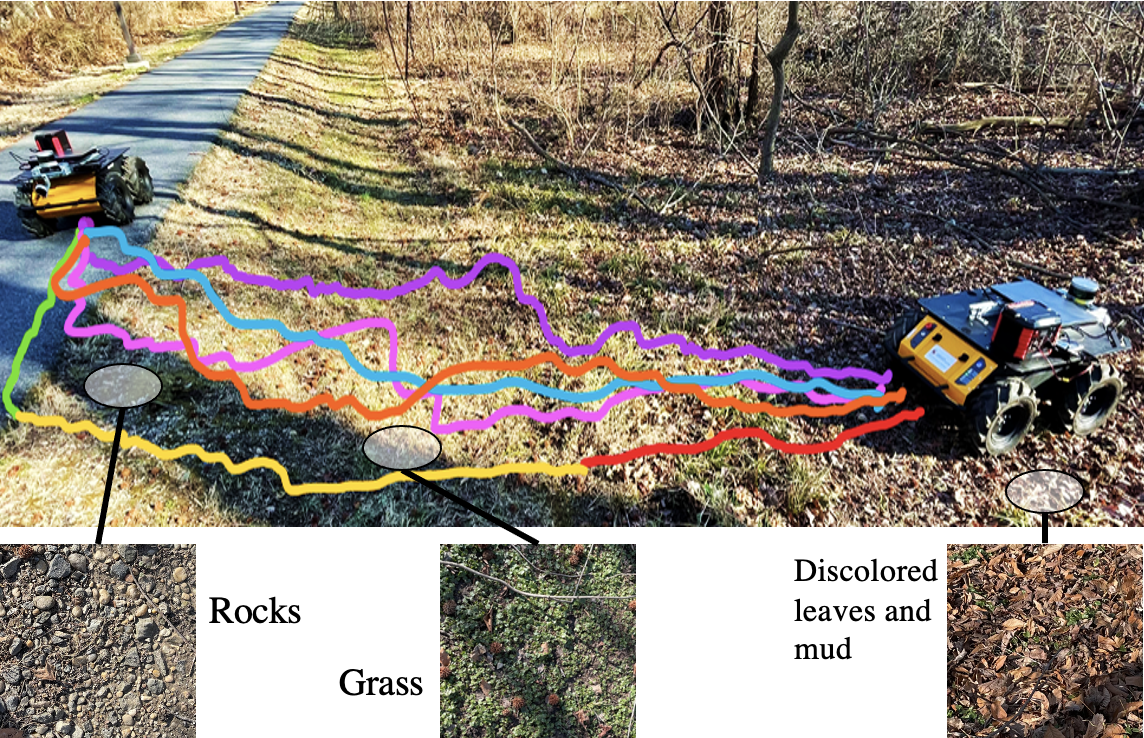}
      \caption {\small{The trajectories of our method, TerraPN (green, yellow, red corresponding to fast, intermediate and slow speeds respectively), DWA \cite{DWA} (blue), TERP \cite{weerakoon2021terp} (orange), OCRNet-based (pink), and PSPNet-based (purple) navigation schemes in unstructured terrain. We observe that TerraPN's trajectories navigate the robot on smooth surfaces (asphalt) as much as possible with high velocities and adapt the velocity to different surfaces based on their navig ability costs. Other methods directly drive the robot towards its goal with its maximum velocity. In some cases (PSPNet-based), the trajectories are wandering since the segmentation falters when the surface becomes too unstructured. 
      }}
      \label{fig:cover-image}
      \vspace{-15pt}
\end{figure}

To overcome the aforementioned limitations, navigability costs could be directly predicted using input and label data collected in the real world through regression using \textit{self-supervised learning} \cite{where-should-i-walk,model-error}. That is, an image (input) can be associated with data vectors collected through other sensors such as Inertial Measurement Units (IMU) and wheel encoders on the robot (labels) instead of a human-provided label/annotation. Such label vectors generated from real-world sensor data could lead to a more accurate characterization of the robot-terrain interaction instead of a human-perceived ground truth. Therefore, using the self-supervised learning paradigm could be used to quickly, and efficiently learn surface characteristics. 

% Additionally, training a network using real-world data instead of simulations would also avoid sim-to-real transfer issues. 

Existing works that use self-supervised learning (online and offline) in the outdoor domain have predominantly focused on unstructured obstacles detection \cite{traverse-classific-unsupervised-online-visual-learning,kahn2020badgr}, roadway and horizon detection \cite{optical-flow-SS}, or long-range terrain classification into various categories \cite{long-range-vision,classifier-ensembles,forested-terrain-SS}. Therefore, there is a lack of online self-supervised learning methods that can characterize a terrain's surface properties. Many previous works also train their network offline by first collecting data before training on a powerful GPU \cite{where-should-i-walk,kahn2020badgr}. As a result, these methods can take a long time to re-train for a new kind of terrain. 

% Additionally, properties such as surface textures, bumpiness, granularity are difficult to simulate with high fidelity, making real-world data collection the best option for learning a terrain's navigability.   

% There have been several self-supervised learning-based works which have dealt with complex outdoor environments. Most works have trained a robot to avoid collisions with outdoor unstructured obstacles \adarsh{add refs} assuming a reasonably flat navigable terrain. Such methods also train the robot offline by first collecting data and training later on a powerful GPU. Therefore, understanding and adapting to a terrain in real-time is not possible. Few works \cite{where-should-i-walk} have performed regression to learn the optimal surfaces to traverse on for a legged robot using force/torque sensors on its feet.    

\textbf{Main Contributions:} We present TerraPN (Terrain cost Prediction and Navigation), a method that uses online self-supervised learning to predict a navigability cost map (or surface cost map) and uses it for efficient robot navigation in outdoor terrains. For training TerraPN's surface cost prediction network, a robot autonomously collects inputs (RGB image patches, robot velocities) and labels (IMU vibrations, and robot's odometry errors) by traversing various surfaces with different velocities. Our network learns the correlations between a surface's visual properties (color, texture, etc.) with its attributes such as traction, bumpiness, and deformability and encodes them in the form of surface costs. TerraPN does not depend on human-annotated datasets and requires minimal human supervision during data collection and training. Using the cost map, TerraPN adapts the Dynamic Window Approach (DWA) \cite{DWA}, a method that guarantees dynamically feasible robot trajectories, for outdoor navigation. The novel aspects of our approach include:

\begin{itemize}

\item A novel method that trains a neural network in an online self-supervised manner to learn a terrain's surface properties (texture, bumpiness, deformability, etc.) and computes a robot-specific 2D surface cost map. The predicted cost map is a concatenation of patches of costs corresponding to different input RGB patches. Our network trains in $\sim 20$ minutes for 5 different surfaces compared to segmentation methods that require 3-10 hours of offline training to achieve similar levels of performances suitable for navigation. Using patches of RGB images leads to a low input dimensional formulation for our network.

\item An algorithm to non-uniformly sample patches from the input image to obtain improved navigability cost representations at the boundaries between surfaces and maintain low inference times. Non-uniform sampling leads to using a lower number of patches per image to predict the costs for the surfaces in a scene compared to uniform sampling. This results in a reduction of 47.27\% in our method's inference time on a mobile commodity GPU.
% \item An algorithm to vary the size of the patches needed to be cropped from the full-sized RGB image to predict the scene's navigability cost map. Our method uses a weak classifier to estimate the number of surfaces in the scene, the regions that predominantly contain a single surface, and crops larger patches from such regions. This leads to a lower number of patches needed for prediction resulting in a reduction in the network's inference time from X to Y on average.

% \item An algorithm to vary the size of the patches needed to be cropped from the full-sized RGB image to predict the scene's navigability cost map. Our method uses a weak classifier to estimate the number of surfaces in the scene, the regions that predominantly contain a single surface, and crops larger patches from such regions. This leads to a lower number of patches needed for prediction resulting in a reduction in the network's inference time.

\item A novel outdoor navigation algorithm that computes dynamically feasible collision-free trajectories by accounting for different surface costs. Our method adapts DWA's \cite{DWA} formulation by modulating the robot's acceleration limits for different surfaces and computing trajectories with lower surface costs compared to DWA's trajectories. This results in improved success rates (up to 35.84\%), vibration costs (up to 21.52\% lower), and mean velocity is high-cost surfaces (up to 46.76\% lower) for different outdoor scenarios. We implement our method on a Clearpath Husky robot and evaluate it in real-world unstructured scenarios with previously unseen surfaces.

% \item An extension to DWA called DWA-O, which accounts for different terrains' surface properties during navigation. DWA-O appropriately modulates the robot's acceleration limits based on a surface's navigability costs and transitions between multiple surfaces. DWA-O also computes velocities with low surface navigability cost for the robot. This leads to smoother trajectories and a reduction in the vibration cost experienced by the robot.
\end{itemize}

We use a learning-based approach for the cost map prediction and a model-based (extending DWA \cite{DWA}) approach for navigation. Therefore, TerraPN combines the benefits of accurate characterization of robot-terrain interaction and guaranteed low surface cost, dynamically feasible navigation.
\section{Related Works}
In this section, we discuss previous works in computer vision for characterizing a terrain's traversability and methods for outdoor navigation.

\subsection{Characterizing Traversability}
Traditional vision works have used methods such as Markov Random fields \cite{prob-terrain-class} and triangular mesh reconstruction \cite{3d-mesh} on 3D point clouds to analyze surface roughness and traversability. Learning-based works for traversability prediction fall into a combination of supervised/unsupervised and classification/regression categories.
% Traditional vision works have used methods such as Markov Random fields \cite{prob-terrain-class} on fused data from 3D lidar and RGB data and triangular mesh reconstruction \cite{3d-mesh} from 3D point clouds to analyze surface roughness and traversability. All previous learning-based works in computer vision for traversability prediction fall into a combination of supervised/unsupervised and classification/regression categories (discussed below).

\subsubsection{Supervised Methods}
Works in pixel-wise semantic segmentation classify a terrain into multiple predefined classes such as traversable, non-traversable, forbidden, etc. \cite{guan2021ganav,ttm,geo-visual}. To this end, fusing a terrain's semantic and geometric features has also been studied \cite{geo-visual}. These works are typically supervised and utilize large hand-labeled datasets of images \cite{rugd, rellis} to train classifiers. However, manually annotating datasets is time- and labor-intensive, not scalable to large amounts of data, and may not be applicable for robots of different sizes, inertias, and dynamics \cite{anomaly-detection}. Such methods also assume that visually similar terrains have the same traversability \cite{traverse-classific-unsupervised-online-visual-learning} without considering the robot's dynamics, velocities, or other constraints.

% Works in pixel-wise semantic segmentation classify a terrain into multiple predefined classes such as traversable, non-traversable, obstacle, forbidden, etc. \cite{guan2021ganav,ttm,geo-visual}. Fusing a terrain's semantic (visual) and geometric (point cloud) features for better classification has also been studied \cite{geo-visual}. These works typically fall under the supervised-classification category and utilize large hand-labeled datasets of images \cite{rugd, rellis} to train classifiers. However, manually annotating datasets is time- and labor-intensive, not scalable to large amounts of data, and may not be applicable for robots of different sizes, inertias, and dynamics \cite{anomaly-detection}. Such works also assume that terrains that are visually similar have the same traversability \cite{traverse-classific-unsupervised-online-visual-learning} without considering the robot's velocities.

\subsubsection{Self-supervised Methods}
Unsupervised learning-based methods overcome the need for such datasets by automating the labeling process by either collecting terrain-interaction data such as forces/torques \cite{where-should-i-walk}, contact vibrations \cite{rover-cmu}, acoustic data \cite{proprioceptive-sensor}, vertical acceleration experienced \cite{multi-sensor-correlation}, and stereo depth \cite{long-range-vision,classifier-ensembles} and associating them with visual features (RGB data) for self-supervision or reinforcement learning \cite{weerakoon2021terp}. Other works have correlated 3D elevation maps and egocentric RGB images \cite{trav-analysis-terrain-mapping} or overhead RGB images \cite{overhead-images} for classification. 

% Early work \cite{long-range-from-short-range,classifier-ensembles} focused on learning long-range ($\sim 100m$) terrain classification from stereo and RGB image data and 3D lidar points \cite{forested-terrain-SS} and detecting roadways and the horizon using reverse optical flow \cite{optical-flow-SS} in a self-supervised manner. Other works have correlated 3D elevation maps and egocentric RGB images \cite{trav-analysis-terrain-mapping} or overhead RGB images \cite{overhead-images} for classification. 

Few methods have performed self-supervised regression, where for each RGB pixel in outdoor terrain, the corresponding force-torque measurements \cite{where-should-i-walk}, trajectory model error \cite{model-error}, and resistive forces experienced \cite{pliable} were predicted post-training. 

% \cite{model-error} presents a method for labeling images with the difference between a robot's actual and predicted trajectories based on its dynamics model. Ordonez et al. \cite{pliable} model the interaction between wheeled/tracked robots with pliable outdoor vegetation by mapping RGB data with the resistive forces experienced by the robot. Our method is complementary to these works, but does not perform pixel-wise regression to reduce computation costs. We choose regression over terrain classification since it does not limit terrain characterization to a set of finite classes. 

\vspace{-5pt}
\subsection{Outdoor Navigation}
Early works in outdoor navigation proposed using the binary classification of obstacles versus free space \cite{Laubach} and potential fields \cite{Shimoda_potential} for outdoor collision avoidance. With the advent of deep learning, methods to estimate navigability/energy cost in uneven terrains through imitation learning \cite{silver2010learning,sriram-siva-1} using egocentric sensor data and a priori environmental information \cite{zakharov2020energy} have been proposed. Siva et al. \cite{siva2021robot} addressed navigational setbacks due to wheel slip and reduced tire pressure in outdoor terrains by learning compensatory behaviors.

% Several works in deep reinforcement learning (DRL) \cite{guastella2021learning,zhang2018robot,Nguyen} have proposed end-to-end systems that use data such as elevation maps, depth images, or raw point clouds for perception along with the robot's pose for training navigation networks. 

% Siva et al. \cite{sriram-siva-1} propose a method to unify representation and imitation learning to estimate important terrain features for robot adaptation in unstructured environments. Subsequently, \cite{siva2021robot} addressed navigational setbacks due to wheel slip and reduced tire pressure by learning compensatory behaviors.

% Add badgr and associated papers and TERPs
BADGR \cite{kahn2020badgr} presents an end-to-end DRL-based navigation policy that learns the correlation between events (collisions, bumpiness, and change in position) and the actions performed by the robot. Since learning-based approaches cannot guarantee any optimality in terms of a navigation metric (minimal cost, path length, dynamic feasibility, etc.), \cite{weerakoon2021terp} proposed a hybrid model of spatial attention for perception and a dynamically feasible DRL method for navigation.
 
% TODO: Mention fusion of learning for perception, and model-based for navigation.
% In TerraPN, we use a learning-based approach for the cost map prediction and a model-based (extending DWA \cite{DWA}) approach for navigation. Therefore, TerraPN combines the benefits of accurate characterization of robot-terrain interaction and guaranteed low surface cost, dynamically feasible navigation.
 
\section{Background and Problem Formulation}
In this section, we define the problem formulation and provide some background on DWA \cite{DWA}.

\begin{figure*}[t]
      \centering
      \includegraphics[width=15.3cm,height=3.5cm]{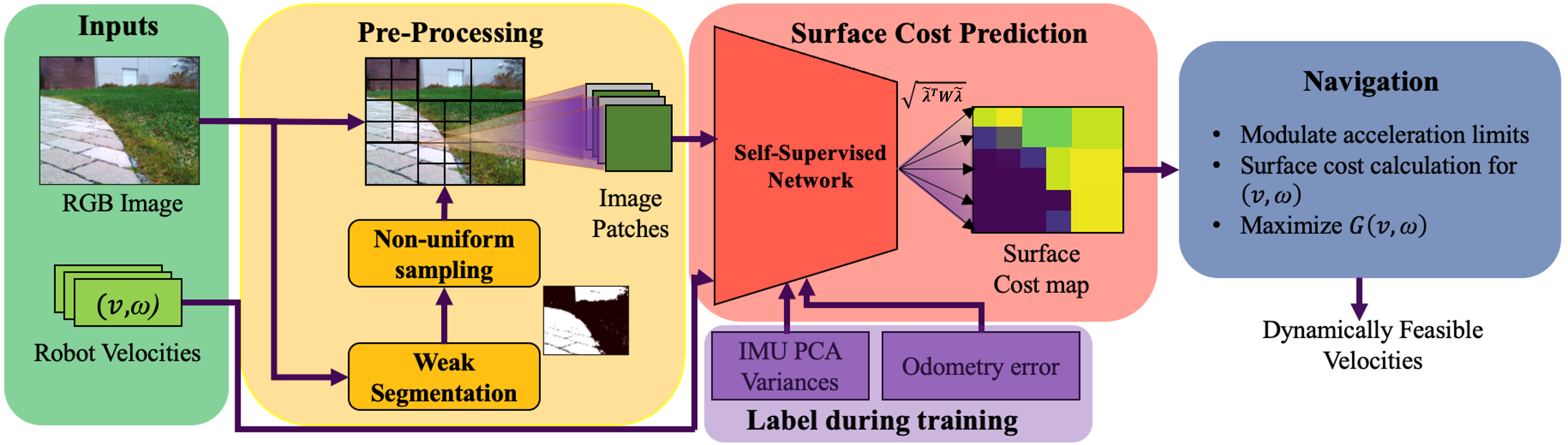}
      \caption {\small{TerraPN's overall architecture. Our method uses RGB images and a set of robot velocities to predict a surface cost map for the robot. In the pre-processing step, a classical image segmentation method (Weak segmentation) is used to differentiate regions with different pixel intensities. A non-uniform patch sampling scheme is then applied where large/small patches are extracted from regions that have the same/varied pixel intensities. The patches are passed into our self-supervised network (see Fig. \ref{fig:network-arch}) that predict a $4 \times 1$ label vector $\Tilde{\boldsymbol\lambda}$. We compute the surface cost as $\Tilde{\boldsymbol\lambda}$'s weighted norm, create patches with the cost values, and concatenate them based on their locations on the RGB image to form the cost map (blue implies low, yellow implies high costs). The cost map is then used by our novel navigation algorithm to compute dynamically feasible velocities with low surface costs in real-time. This velocity computation task is performed iteratively at every time step for a given image and velocity vector input pair.   
      }}
      \label{fig:system_architecture}
\end{figure*}

\subsection{Problem Formulation} \label{sec:prob-form}
The problem of navigating in outdoor environments with various surface properties can be divided into two stages: surface cost prediction, and low-cost navigation. To predict surface costs, we train a neural network that uses RGB image patches $I^{n \times n}_{RGB}$, a set of the robot's linear and angular velocities ($v, \omega$) as inputs, and a $4 \times 1$ label vector $\boldsymbol\lambda$ consisting of two dimensions corresponding to IMU measurements, and two to robot's odometry errors. The IMU measurements and odometry errors characterize a surface's traction, bumpiness, and deformability (see Section \ref{sec:comp-input-labels}). After training, we use the neural network's predicted label $\Tilde{\boldsymbol\lambda}$ to construct a 2D surface cost map. The cost map is a non-uniformly discretized concatenation of patches that depend on the surface boundaries in the input image.

% The cost map differentiates various surfaces based on their degree of navigability depending on their properties and the robot's velocities and dynamics.

% Navigability prediction is formulated as the problem of learning a function $\mathcal{F}(I^{n \times n}_{RGB}, (\mathbf{v}, \boldsymbol\omega)) = \Tilde{\mathbf{l}}$. Where, $I^{n \times n}_{RGB}$ refers to an $n \times n$ patch sampled from an RGB image, $(\mathbf{v}, \boldsymbol\omega)$ denotes vectors of linear and angular velocities executed by the robot for the past $n/2$ time instances, and $\Tilde{\mathbf{l}}$ denotes a predicted label vector which characterizes a surface's traction, bumpiness and deformability. 

% For predict these costs, a neural network is trained to perform a self-supervised regression task, i.e., learn a terrain's surface properties by correlating its two inputs (RGB images and the robot's velocity data) with the vibrations and the odometry errors experienced by the robot while traversing it. 

Next, the cost map is used by the navigation component to compute dynamically feasible trajectories with low surface costs. We use \textit{i, j} for denoting various indices, $x, y$ to denote positions relative to different coordinate frames. Our method's overall architecture is shown in \ref{fig:system_architecture}.

% CBR
% Our formulation includes the robot velocities as inputs since they have a direct correlation with the level of vibrations and odometry errors experienced by the robot on any surface. For instance, a robot with higher velocities on a rough terrain would lead to high levels of vibrations. On a surface like sand (poor traction) where the robot's wheels could get stuck, high velocities would lead to high position and orientation errors. 

% \begin{table}[]
% \centering
% \begin{tabularx}{\linewidth}{|c|L|} % Change the alignment to center
% \hline
% \textbf{Symbols} & \textbf{Definitions}  \\
% \hline
% $v, \omega, \Dot{v}, \Dot{\omega}$ & Linear and angular velocities and accelerations respectively. \\
% \hline
% $v_{max}, \omega_{max}, \Dot{v}_{max}, \Dot{\omega}_{max}$ & Robot's maximum linear, angular velocities and accelerations respectively.\\
% \hline
% $I^{n \times n}_{RGB}, I^{w \times h}_{RGB}$ & RGB image patch and full-sized image. \\
% \hline
% $C^{w \times h}$ & Predicted navigability cost map for $I^{w \times h}_{RGB}$. \\
% \hline
% $\Dot{v}_{lim}, \Dot{\omega}_{lim}$ & Linear and angular acceleration limits computed by DWA-O. \\
% \hline
% \end{tabularx}
% \caption{\small{List of frequently used symbols in our approach.}}
% \label{tab:symbol_defn}
% \end{table}

\subsection{Dynamic Window Approach}
DWA is a model-based collision avoidance algorithm that guarantees dynamically feasible robot velocities. However, its formulation implicitly assumes that the robot traverses on a uniform, smooth navigable surface. TerraPN's navigation component modifies and extends DWA's formulation for outdoor navigation by accounting for surface costs. 

DWA's formulation involves two stages: 1. computing a collision-free, dynamically feasible velocity search space, and 2. choosing a velocity in the search space to maximize an objective function. 

In the first stage, all possible velocities in $V_s = [v \in [0, v_{max}], \omega \in [-\omega_{max}, \omega_{max}]$, respectively, are considered for the search space $V_s$. Here, $v_{max}, \omega_{max}$ represent the robot's maximum achievable linear and angular velocities respectively. Next, all the $(v, \omega)$ pairs that prevent a collision in $V_s$ are used to form the admissible velocity set $V_a$. Lastly, the velocity pairs that are reachable, accounting for the robot's acceleration limits $\Dot{v}_{max}, \Dot{\omega}_{max}$ within a short time interval $\Delta t$, are considered to construct a dynamic window set $V_d$. The resulting search space is constructed as $V_r = V_s \cap V_a \cap V_d$. 

In the second stage, DWA searches for $(v, \omega) \in V_r$, that maximizes the following objective function. 

\vspace{-4pt}
\begin{equation}
    G_1(v, \omega) =  \alpha.head + \beta.dist +  \gamma.vel
    \label{eqn:dwa-obj-func}
\end{equation}
% The objective function is a weighted sum of three terms: 1. $head(v, \omega)$, 2. $dist(v, \omega)$, and 3. $vel(v, \omega)$. 

Here, \textit{head()} measures the progress towards the robot's goal, \textit{dist()} is the distance to the closest obstacle when executing a certain $(v, \omega)$, and \textit{vel()} measures the forward velocity of the robot and encourages higher velocities. $v, \omega$ on the RHS are omitted for clarity.
\section{TerraPN: Surface Cost Prediction}
In this section, we describe the components in computing a terrain's surface cost map: 1. data collection, 2. network architecture and training, 3. cost map generation using variable sampling.

\subsection{Autonomous Data Collection}
To generate the input and label data for training the cost map prediction network (Fig. \ref{fig:network-arch}), we collect the raw sensor data from an RGB camera, robot's odometry, 6-DOF IMU, and 3D lidar autonomously on different surfaces. The robot performs a set of maneuvers in two different speed ranges: 1. slow ([0, $\frac{v_{max}}{2}$] m/s and [-$\frac{\omega_{max}}{2}$, $\frac{\omega_{max}}{2}$] rad/s), and 2. fast ([0, $v_{max}$] m/s and [$-\omega_{max}$, $\omega_{max}$] rad/s). The maneuvers include: 1. moving along a rectangular path, 2. moving in a serpentine trajectory, and 3. random motion. 

The maneuvers are designed to cover all the $[v, \omega]$ pairs within the robot's velocity limits to emulate different terrain interactions while the data is collected. If the robot encounters an obstacle, it switches from executing the maneuvers to avoiding a collision using DWA \cite{DWA}. 

\subsection{Computing Inputs and Labels} \label{sec:comp-input-labels}
As mentioned in section \ref{sec:prob-form}, our network's uses an $n \times n$ patch ($I_{RGB}^{n \times n}$) that is cropped from the center bottom of the collected full-sized image of size $w \times h$ ($I_{RGB}^{w \times h}$). To generate the velocity input, the linear and angular velocities for the past $n/2$ instances from the robot's odometry are obtained and reshaped to a $2 \times n/2$ vector. The velocity vector's dimensions are chosen such that the image input does not dominate the network's predictions.

Our label vector consists of IMU and odometry error components. They are robot-specific and implicitly encode the robot's dynamics, inertia, etc., and its interactions with the terrain. To generate the IMU component, we apply Principal Component Analysis (PCA) to reduce the dimensions of the collected 6-dimensional IMU data (linear accelerations and angular velocities) to its two principal components. From Fig. \ref{fig:pca}, we observe that the variances ($\sigma_{PC1}$ and $\sigma_{PC2}$) of the data along the principal components help differentiate various surfaces in terms of their bumpiness. Additionally, for the same surface, higher velocities lead to higher variances in the data (justifying the need to consider velocities as inputs). 
% \vspace{-5pt}
\begin{figure}[t]
      \centering
      \includegraphics[width=7cm,height=2.0cm]{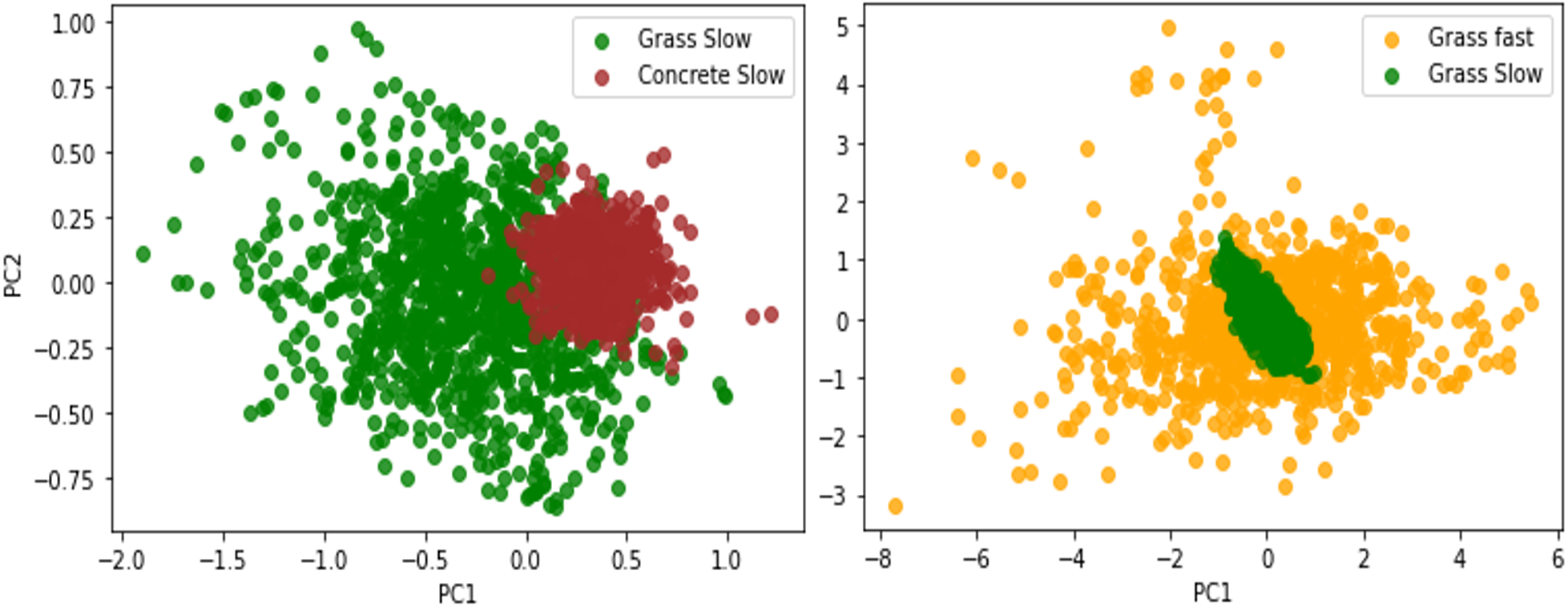}
      \caption {\small{Sample results of the PCA applied on the 6-dimensional IMU data. 
      The variances of the data along the two principal axes for \textbf{[Left]} two different surfaces concrete (red), and grass (green), and \textbf{[Right]} two different velocity ranges for the same surface. We observe that the variances can differentiate various surfaces and speed levels.
    %   \textbf{[Left]} We observe that the variances of the data along the two principal axes for concrete (red), and grass (green) for the same set of robot velocities are considerably different. \textbf{[Right]} The variances for the slow (green) and fast (yellow) velocities for the same surface (grass) are considerably different. Therefore, variances can clearly differentiate various surfaces and speed levels.
      }}
      \label{fig:pca}
\end{figure}

To generate the odometry error component of the label vector, the distance traveled by the robot ($\Delta d_{odom}$) and its change in orientation ($\Delta\theta_{odom}$) in a time interval $\Delta t$ are obtained from the robot's odometry. We obtain the same data ($\Delta d_{loam}$, $\Delta\theta_{loam}$) from a 3D lidar-based odometry and mapping system \cite{legoloam}. The distance and orientation change errors are calculated as,
\vspace{-8pt}
\begin{gather}
    d_{error} = \Delta d_{loam} - \Delta d_{odom} \\ 
    \theta_{error} = \Delta \theta_{loam} - \Delta \theta_{odom}.
\end{gather}
% \vspace{-5pt}
\no This component of the label vector differentiates surfaces with high deformability or poor traction where, if the robot's wheels get stuck or slip (e.g. in mud), $\Delta d_{loam} \approx 0$ and $\Delta \theta_{loam} \approx 0$, whereas $\Delta d_{odom}$ and $\Delta \theta_{odom}$ would have high values. The final label vector is given by,
\vspace{-8pt}
\begin{equation}
    \boldsymbol\lambda = [\sigma_{PC1} \quad \sigma_{PC2} \quad d_{error}  \quad \theta_{error}]^\top.
    \label{eqn:label-vector}
\end{equation}

\subsection{Network Architecture and Online Training}
The network (see Fig. \ref{fig:network-arch}) is trained to predict the $4 \times 1$ vector in equation \ref{eqn:label-vector} given the image and velocity inputs. The architecture uses a series of 2D convolution with skip connections and batch normalization on the image input, and several layers of fully connected layers with dropout and batch normalization for the velocity input, shown in Fig. \ref{fig:network-arch}. In the image stream, after the initial convolution operation with ReLU activation, the image is connected to four residual blocks and one linear layer with batch normalization. Since we have limited data when collecting labels and training the network online, we use the residual connection to ensure that the network would not overfit the collected data given many layers and parameters in the network. Dropout and batch normalization layers also improve generalization capabilities and avoid overfitting. 

% Since the size of the image patches used is low, we design a light-weight model that has 64 channels in the first two residual blocks and 128 channels in the last two blocks. We find the best performance after adding dropout and batch normalization layers to avoid overfitting. 
% \vspace{-5pt} 
\begin{figure}[t]
      \centering
      \includegraphics[width=\columnwidth]{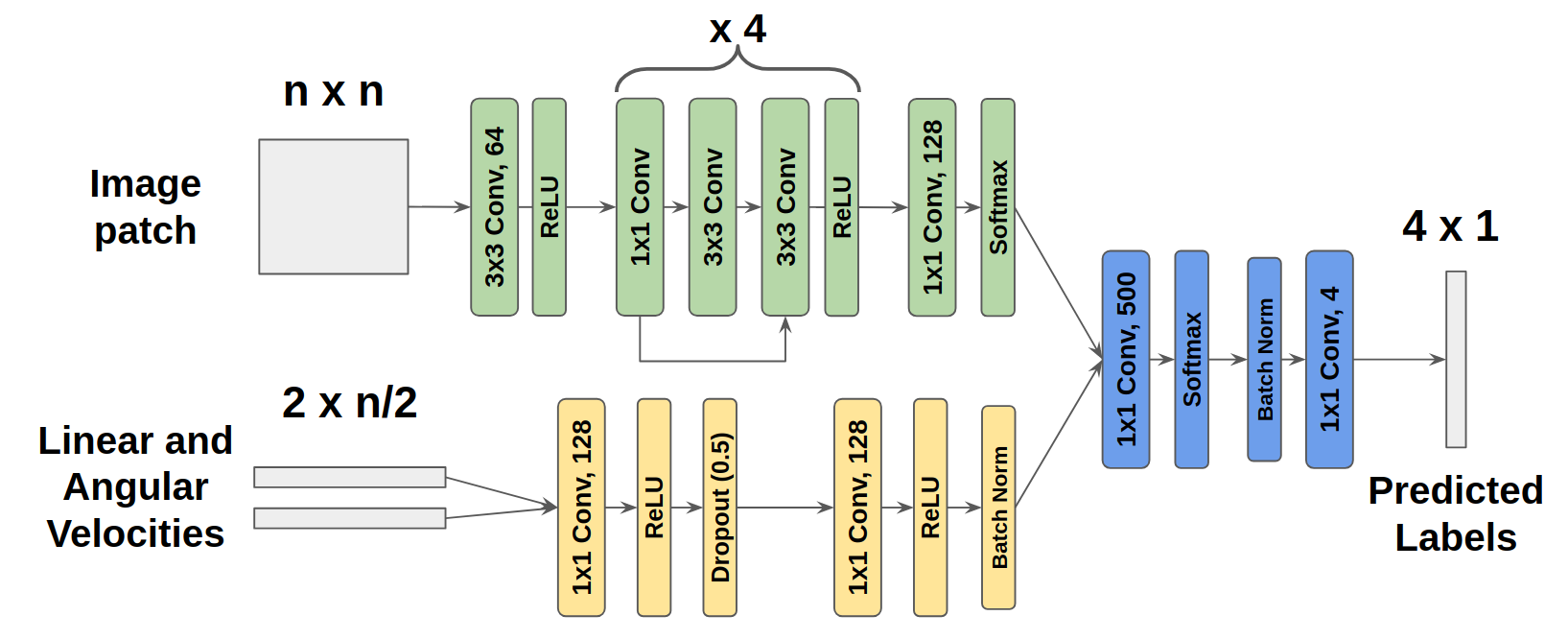}
      \caption {\small{Our novel two-stream network architecture. One stream encodes image patches (green) and the other stream processes the linear and angular velocities (yellow). The feature embeddings from each stream are concatenated and finally passed into a set of linear layers (blue), and the final predictions are the estimated IMU readings and the odometry errors associated with this image patch. }}
      \label{fig:network-arch}
\end{figure}

% Online training
As the robot autonomously collects sensor data on different surfaces, the inputs and labels are generated and shuffled. The network's training is started and performed online once sufficiently varied data is collected ($\sim 6$ minutes). The data collection and online training is normally completed in $20-25$ minutes.

\subsection{Navigability Cost Calculation}
Based on the vector predicted by the network ($\Tilde{\boldsymbol\lambda} = [\Tilde{\sigma}_{PC1} \quad \Tilde{\sigma}_{PC2} \quad \Tilde{d}_{error}  \quad \Tilde{\theta}_{error}]^\top$), the navigability cost for a given RGB patch and velocity vector is computed as the weighted norm of $\Tilde{\boldsymbol\lambda}$,

\vspace{-10pt}
\begin{equation}
    % c = \sqrt{\sum_{i=1}^{4} w_i \cdot \Tilde{l}_i^2}.
    % c_{ij} = \sqrt{\mathbf{w} \cdot \Tilde{\mathbf{l}}}.
    c = \sqrt{\Tilde{\boldsymbol\lambda}^\top \mathbf{W} \Tilde{\boldsymbol\lambda}}
\end{equation}

Here, $\mathbf{W}$ is a diagonal matrix with positive weights. To make navigability cost predictions on a full-sized RGB image $I^{w \times h}_{RGB}$, the image is first resized into new dimensions as follows,
\vspace{-10pt}
\begin{equation}
    w' = \nint{w/n} \cdot n, \quad h' = \nint{h/n} \cdot n,
\end{equation}

\no where $\nint{}$ is the nearest integer operator. Next, non-overlapping $n \times n$ patches are cropped along the width and height of the image. This $\nint{w/n} \cdot \nint{h/n}$ batch of images is passed as inputs along with a batch containing the input velocity vectors to obtain the navigability cost predictions $c_{ij}$ $(i \in [1, \nint{w/n}]$ and $j \in [1, \nint{h/n}])$ for different regions of the resized image $I^{w' \times h'}_{RGB}$. The predicted costs are normalized to be in the range $[0, \pi/2]$. 

Finally, the surface navigability cost map $C^{w' \times h'}$ is constructed by vertically and horizontally concatenating $n \times n$ patches with the values of $c_{ij}$ corresponding to different regions. $C^{w' \times h'}$ is then resized back to $C^{w \times h}$.

\subsection{Non-uniform Patch Sampling}
% To efficiently manage the time required to compute the navigability cost in a full-sized image, we sample it non-uniformly. 
Although using patches reduces the input dimensionality of our network and helps train it faster, it could result in regions with multiple surfaces (boundaries) in the full-sized image having inaccurate costs depending on the patch size. Therefore, we employ a non-uniform patch sampling technique to obtain finer patches in multi-surface regions (boundaries). Conversely, portions with a single surface can be sampled with larger patch sizes $2n, 4n$. This reduces the total number of patches used for cost prediction on the full-sized image, thus maintaining the method's inference rate. 

% For non-uniform sampling, first the number of surfaces in a scene must be estimated.

% The following section can be written as an algorithm instead
\subsubsection{Weak Segmentation} 
To detect regions with multiple surfaces and differentiate them in a scene, we use a \textit{weak} segmentor which is based on classical image segmentation methods. The weak segmentation method is computationally light (inference time of $\sim 66 \mu s$) and may not have pixel-level precision. But, it sufficiently demarcates the boundaries of various terrains in the input image. First, the Sobel edge detector is applied to the grayscale input image $I^{w' \times h'}_{gray}$ of the scene, and the histogram of the result is computed. Next, based on the Bayesian information criterion \cite{schwarz1978estimating} a Gaussian mixture model is fit to the histogram, and the mean of each Gaussian curve is used as a marker/threshold ($\mu_1, \mu_2, ... , \mu_k$) to differentiate the regions of pixels with different intensity levels in the image. Finally, the watershed filter \cite{watershed} is applied to highlight the regions of different intensities to obtain $I^{w \times h}_{WS}$ (See figure \ref{fig:system_architecture}). 

\subsubsection{Selecting Sampling Patch Size} 
We consider the patches $I^{n \times n}_{WS}, I^{2n \times 2n}_{WS}, I^{4n \times 4n}_{WS}$ in the weak segmentation output $I^{w' \times h'}_{WS}$ and the intensity of pixels within them. If a patch larger than $n \times n$ satisfies the condition in \ref{eqn:patch-division}, smaller patches are not considered for cost prediction.
\vspace{-4pt}
\begin{equation}
    % num(\mu_i) > \tau \cdot num(\mu_j) \qquad i, j \in [1, k]. 
    \frac{num(\mu_i)}{n^2} > \xi \qquad i \in [1, k].
    \label{eqn:patch-division}
\end{equation}

where $num(\mu_i)$ is the number of pixels with intensity greater than or equal to $\mu_i$, and $\xi$ is a threshold. This condition ensures that when a large patch has a significant number of pixels with the same intensity, implying the presence of a single surface, smaller patches are not used for cost prediction. The larger patch is resized to $n \times n$ before passing into the network. 
\section{TerraPN: Navigation}
In this section, we explain how the computed surface cost map is used by TerraPN's navigation to adapt DWA for outdoor navigation and compute dynamically feasible, low surface cost velocities.

\subsection{Trajectory Navigability Cost}
To adapt DWA to outdoor terrains, the trajectory corresponding to a $(v, \omega)$ pair must be associated with a surface cost. The trajectory for a given $(v, \omega)$ pair relative to a coordinate frame attached to the robot's center (X-axis pointing forward, Y-axis pointing left) is calculated as,
\vspace{-5pt}
\begin{equation}
    \begin{split}
        x^{rob}_i &= v\cos{(\omega t_i)}t_i \\
        y^{rob}_i &= v\sin{(\omega t_i)}t_i \\ 
        t_i &= t_0 + i \Delta t, \qquad i \in [0, s_{num}].
    \end{split}
    \label{eqn:traj-compute}
\end{equation}
Here, $t_0$ is the initial time instant and $s_{num}$ is the number of time steps used for extrapolating the trajectory. This trajectory is then transformed relative to the camera frame attached to the robot using a homogeneous transformation matrix as $[x^{cam}_i \,\, y^{cam}_i]^\top = H^{cam}_{rob} \cdot [x^{rob}_i \,\, y^{rob}_i]^\top$. Next, the trajectory is converted to correspond to the image/pixel coordinates of $I^{w' \times h'}_{RGB}$, i.e., $[x^{img}_i \,\, y^{img}_i]^\top$ using the camera's intrinsic parameters. The navigability cost for a velocity pair can be then computed from $C^{w' \times h'}$ as, 
\vspace{-10pt}
\begin{equation}
    sur(v, \omega) = \sum_{i=0}^{s_{num}} cost(x^{img}_i, y^{img}_i).
\end{equation}

\no Here, \textit{cost()} is the surface cost at a given pixel's coordinates.
% TODO: Mention about normalizing the sur() to 0-pi/2 somewhere. 

\subsection{Variable Acceleration Limits}
Robot navigation methods consider a constant range of linear ($[-\Dot{v}_{max}, \Dot{v}_{max}]$) and angular ($[-\Dot{\omega}_{max}, \Dot{\omega}_{max}]$) accelerations. Our formulation varies the linear and angular acceleration limits available for planning depending on the properties of the surface on which the robot is traversing such that $\Dot{v} \in [-\Dot{v}_{max}, \Dot{v}_{lim}]$ and $\Dot{\omega} \in [-\Dot{\omega}_{lim}, \Dot{\omega}_{lim}]$.

This is done because, intuitively, the robot accelerating on a smooth surface (e.g., concrete, asphalt) would lead to a low navigability cost. Therefore, the robot can proceed towards its goal faster. Whereas on a bumpy surface or one with poor traction, (e.g., tiled surface, dry leaves), accelerating would lead to high vibrations and the risk of getting stuck (e.g. mud). We do not limit the maximum deceleration available to the robot since it may have to slow down to avoid obstacles or while moving on a rough surface.

First, we divide the trajectory corresponding to the robot's current $(v, \omega)$ and  calculate the cost for the second half as follows,
\vspace{-10pt}
\begin{equation}
    \mathcal{C}_{\frac{s_{num}}{2}+1:s_{num}} = \sum_{i=s_{num}/2 + 1}^{s_{num}} cost(x_i^{img}, y_i^{img}) / (s_{num}/2).
    \label{eqn:traj-sur-cost}
\end{equation}

\no We limit the robot's accelerations using this navigability cost as,
\vspace{-15pt}
\begin{gather}
    \Dot{v}_{lim} = \tau \cdot \Dot{v}_{max}, \\
    \Dot{\omega}_{lim} = \tau \cdot \Dot{\omega}_{max}, \\
    \tau = \cos{(\mathcal{C}_{\frac{s_{num}}{2}:s_{num}})}, \quad \mathcal{C}_{\frac{s_{num}}{2}:s_{num}} \in [0, \pi/2].
\end{gather}

\no If $\mathcal{C}_{\frac{s_{num}}{2}+1:s_{num}}$ is low (low-cost surface), the robot is allowed to accelerate towards its goal, while a high $\mathcal{C}_{\frac{s_{num}}{2}+1:s_{num}}$ restricts the robot from speeding up. Considering only the second half of the trajectory also implicitly accounts for transitions between surfaces. Using these acceleration limits, a new dynamic window $V'_d$ is constructed. The new resultant search space is calculated as $V'_r = V_s \cap V_a \cap V'_d$.

% We devise conditions to find out if the robot continues to traverse on the current surface or transitions to another surface within the next $s_{num}$ time steps. First, we divide the robot's current trajectory in two, and calculate the difference in cost for the two halves as,

% \begin{equation}
%     \mathcal{C} = cost_{\frac{s_{num}}{2}:s_{sum}} - cost_{0:\frac{s_{num}}{2}}.
% \end{equation}

% \no Where $cost_i:j$ is the sum of the costs of the points from point i to point j. The conditions for staying and transitioning between surfaces are,
% \begin{equation}
%     \begin{split}
%         \mathcal{C} \approx 0 &\implies \text{\small{Stay on the current surface}}, \\
%         \mathcal{C} >/< 0 &\implies \text{\small{Transition to higher/lower-cost surface}}. \\
%         % \mathcal{C} &< 0 \implies \text{Robot transitioning to lower-cost surface}.
%     \end{split}
%     \label{eqn:cost-conditions}
% \end{equation}

% If the first condition in equation \ref{eqn:cost-conditions} is satisfied, the maximum acceleration of the robot is limited as,

% \begin{gather}
%     \Dot{v}_{lim} = \tau \cdot \Dot{v}_{max}, \\
%     \Dot{\omega}_{lim} = \tau \cdot \Dot{\omega}_{max}, \\
%     \tau = \cos{(sur(v_{curr}, \omega_{curr}))}, \quad sur(v, \omega) \in [0, \pi/2].
% \end{gather}

\vspace{-5pt}
\subsection{Optimization}
Finally, a $(v, \omega) \in V'_r$  which maximizes the objective function
\vspace{-10pt}
\begin{multline}
    % G_2 (v, \omega) =  S( \alpha.head \cdot (1 - \delta.sur)) + \beta.dist +  \gamma.vel,
    G_2 (v, \omega) =  \alpha.head + \beta.dist +  \gamma.vel - \delta.sur,
    \label{eqn:our-obj-func}
\end{multline}

\no is chosen. Here, $\alpha, \beta, \gamma, \delta$ are weights for each component.

\begin{lemma}
TerraPN's navigation computes collision-free, dynamically feasible trajectories with surface costs that are always lesser than or equal to the  DWA's trajectory's surface cost. 
\end{lemma}
\begin{proof}
Let $(v_1, \omega_1) = argmax \, (G_1)$, and $(v_2, \omega_2) = argmax \, (G_2)$, $\forall v, \omega \in V'_{r}$. From equations \ref{eqn:dwa-obj-func} and \ref{eqn:our-obj-func}, we get $G_2(v,\omega) = G_1(v,\omega) - \delta sur(v,\omega)$. We know that, $G_2(v_2, \omega_2) \ge G_2(v_1, \omega_1)$  $\implies G_1(v_2, \omega_2) - \delta sur(v_2, \omega_2) \ge G_1(v_1, \omega_1) - \delta sur(v_1, \omega_1)$. Rearranging the terms, we get $G_1(v_1, \omega_1) - G_1(v_2, \omega_2) \le \delta(sur(v_1, \omega_1) - sur(v_2, \omega_2))$.

% \implies G_2(v,\omega) \le G_1(v,\omega) \quad \forall v,\omega \in V_s$. 
 
Since $v_1, \omega_1$ maximizes $G_1$, in the LHS, $G_1(v_1, \omega_1) - G_1(v_2, \omega_2) \ge 0 \implies sur(v_1, \omega_1) - sur(v_2, \omega_2) \ge 0$. This inequality holds since $sur(v, \omega) \ge 0$. The dynamically feasibility of TerraPN's velocities follows from the fact that acceleration limits obey $[-\Dot{v}_{max}, \Dot{v}_{lim}] \subseteq [-\Dot{v}_{max}, \Dot{v}_{max}]$ as $\tau \in [0, 1]$. 
\end{proof}

\subsection{Algorithm}
\no The algorithm for TerraPN's outdoor navigation is shown in Algorithm \ref{alg:Terrapn-Nav}. 

\begin{algorithm}
	\caption{TerraPN Navigation}
	\hspace*{\algorithmicindent} \textbf{Input} obs, goal, $\dot{v}_\text{max}$, $\dot{\omega}_\text{max}$, $v_\text{curr}$, $\omega_\text{curr}$, $s_\text{num}$\\
    \hspace*{\algorithmicindent} \textbf{Output} $v^*$, $\omega^*$
	\begin{algorithmic}[1]
	    \State $V_s = [[0,v_\text{max}],[-\omega_\text{max}, \omega_\text{max}]]$
	    \State $V_A = $ CollisionFree$(V_s, obs)$
	    \State $\text{traj}_\text{curr} = \text{ComputeTrajectory}(v_\text{curr}, \omega_\text{curr}, s_\text{num})$
	    \State $C_{\frac{s_\text{num}}{2}+1:s_\text{num}} = \text{SurfaceCost}(\text{traj}_\text{curr}, s_\text{num})$
	    \State $\tau = \cos(C_{\frac{s_\text{num}}{2}+1:s_\text{num}})$
	    \State $\dot{v}_\text{lim}, \dot{\omega}_\text{lim} = \tau \cdot \dot{v}_\text{max}, \tau \cdot \dot{\omega}_\text{max}$
	    \State $V_d' = [[v_\text{curr} - \dot{v}_\text{max} \Delta t, v_\text{curr} + \dot{v}_\text{lim} \Delta t], [\omega_\text{curr} - \dot{\omega}_\text{lim} \Delta t, \omega_\text{curr} + \dot{\omega}_\text{lim} \Delta t]]$
	    \State $V'_r = V_s \cap V_a \cap V'_d$
	    \State $G_2 = \alpha.head(V'_r, \text{goal}) + \beta.dist(V'_r, \text{obs}) +  \gamma.vel(V'_r) - \delta.sur(V'_r)$
	    \State $v^*, \omega^* = \underset{v,\omega \in V_r'}{\text{argmax}}(G_2)$
	\end{algorithmic} 
	\label{alg:Terrapn-Nav}
\end{algorithm}

Here, obs is the set of obstacles in the environment characterized by lidar scans. In line 2, CollisionFree() is a function which returns the velocities which avoid a collision with the obstacles in the environment. The ComputeTrajectory() function in line 3 implements equation \ref{eqn:traj-compute}, and the SurfaceCost() in line 4 implements equation \ref{eqn:traj-sur-cost}.   
\section{Results and Evaluations}
We detail our method's implementation, our evaluation metrics, the different environments we tested in, and compare with other methods in this section.

\subsection{Implementation}
Our self-supervised learning network is implemented using Tensorflow. A Clearpath Husky robot mounted with a VLP-16 lidar and Intel Realsense camera is used to evaluate our method in real-world scenarios. The lidar is only used for the initial lidar-based odometry data collection and to provide 2D scans for detecting obstacles for navigation. For processing, the robot is equipped with a laptop with an Intel i9 processor and Nvidia RTX2080 GPU. 

In our formulation, we use n = 50, w = 640, h = 480, $\xi = 0.8$, ${s_{num}} = 15$, $\Delta t = 0.1$, $\alpha = 2.4, \beta = 3.2, \gamma = 0.1, \delta = 50$. The velocity and acceleration limits are set according to the Husky's limits. We train our cost map prediction on 5 surfaces: 1. concrete, 2. tiles, 3. grass, 4. asphalt, and 5. fallen yellow leaves. Each surface offer different levels of traction, bumpiness and deformability.

\subsection{Navigation Evaluation Metrics}
\no We use the following metrics for our comparisons. \\
\no \textbf{Success Rate} - The number of times the robot reached its goal while avoiding getting stuck or colliding over the total number of trials.

\no \textbf{Normalized Trajectory length} - The robot's trajectory length normalized over the straight-line distance to the goal for all the successful trials. 

% \no \textbf{Goal Heading Deviation} - The cumulative angle difference between the robot's heading and the goal along a trajectory.

\no \textbf{Vibration Cost} - The summation of the vertical motion's gradient experienced by the robot along its trajectory.

\no \textbf{Mean Velocity} - The robot's average velocity along its trajectory as it traverses various surfaces.

We compare our method with DWA \cite{DWA}, and TERP \cite{weerakoon2021terp} which navigates the robot based on the elevation maps in the environment. We also compare with two methods that use OCRNet \cite{ocrnet}, and PSPNet \cite{pspnet} for semantic segmentation and the Dijkstra's algorithm for waypoint computation and navigation. OCRNet and PSPNet classify different surfaces into several discrete classes based on their degree of navigability. Different cost values are associated with each class and used for waypoint computation using Dijkstra's algorithm.  The OCRNet and PSPNet networks were trained on the RELLIS-3D outdoor dataset which contains all the surfaces TerraPN's cost prediction network is trained for.

\subsection{Testing Scenarios}

\begin{figure*}[t]
      \centering
      \includegraphics[width=\linewidth]{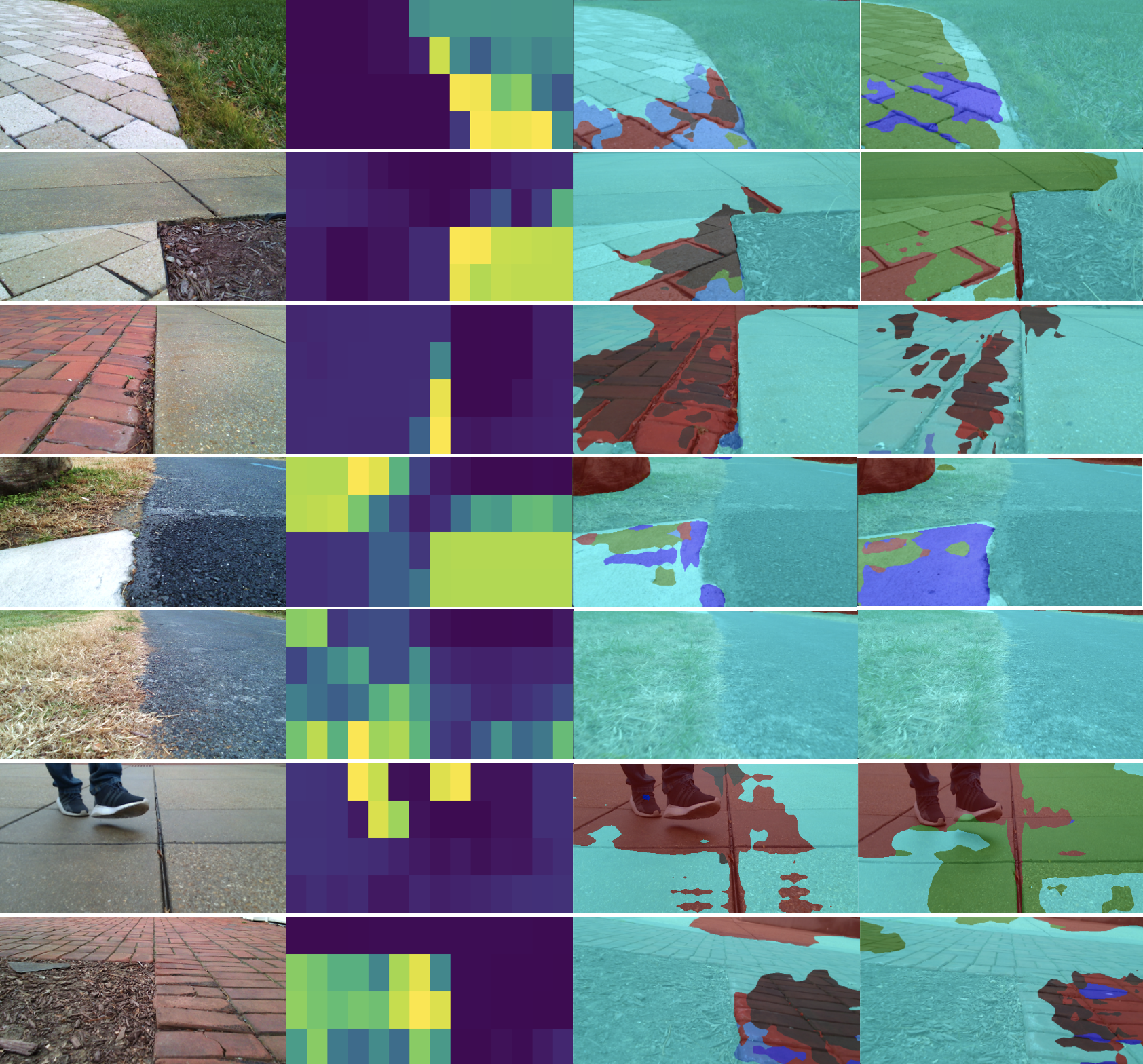}
      \caption {\small{Surface cost map predictions along with OCRNet's (third column) and PSPNet's (fourth column) segmentation results for different scenarios with 2-3 surfaces. TerraPN's cost prediction consider non-uniformly sampled patches on the RGB images and feed them into our network along with the robot's velocities. We observe that surfaces are differentiated well based on their navigability (dark blue denotes low cost and therefore better navigable surface and lighter, yellow shades denote high cost). We observe that our cost predictions are fairly accurate even for regions not seen during training (person's legs, soil, boundaries between surfaces). OCRNet and PSPNet mostly classify surfaces as traversable (green), or moderately traversable (cyan). However, they often incorrectly segment some patches as forbidden (red) or as obstacles (blue) which reflects poorly during navigation.
      }}
      \label{fig:cost-map-predict}
      \vspace{-5pt}
\end{figure*}

% Can characterize based on % of high-cost surface in the trajectory.
We evaluate and compare TerraPN's navigation with prior methods in five scenarios. We characterize each scenario's difficulty based on the number of surfaces, the level of unstructuredness, and whether they were previously seen during TerraPN's cost prediction training.  \\
\no \textbf{Scenario 1:} Two trained surfaces (concrete and grass). See Fig. \ref{fig:traj_comparison}a.

\no \textbf{Scenario 2:} Three surfaces (concrete, asphalt and rocks) with one untrained surface (rocks). See Fig. \ref{fig:traj_comparison}b.

\no  \textbf{Scenario 3:} Four surfaces (tiles, concrete, mud, grass) with one untrained surface (mud) where the robot could get stuck. See Fig. \ref{fig:traj_comparison}c.

\no \textbf{Scenario 4:} Four surfaces (asphalt, rocks, discolored grass, unstructured dry brown leaves with mud). The grass and dry leaves surfaces had undergone considerable seasonal changes. See Fig. \ref{fig:cover-image}.

\no \textbf{Scenario 5:} Untrained, highly unstructured surface with dry brown leaves, broken branches, mud, etc. See Fig. \ref{fig:traj_comparison}d.

\subsection{Analysis and Comparisons}
\begin{figure*}[t]
      \centering
      \includegraphics[width=17cm,height=3.8cm]{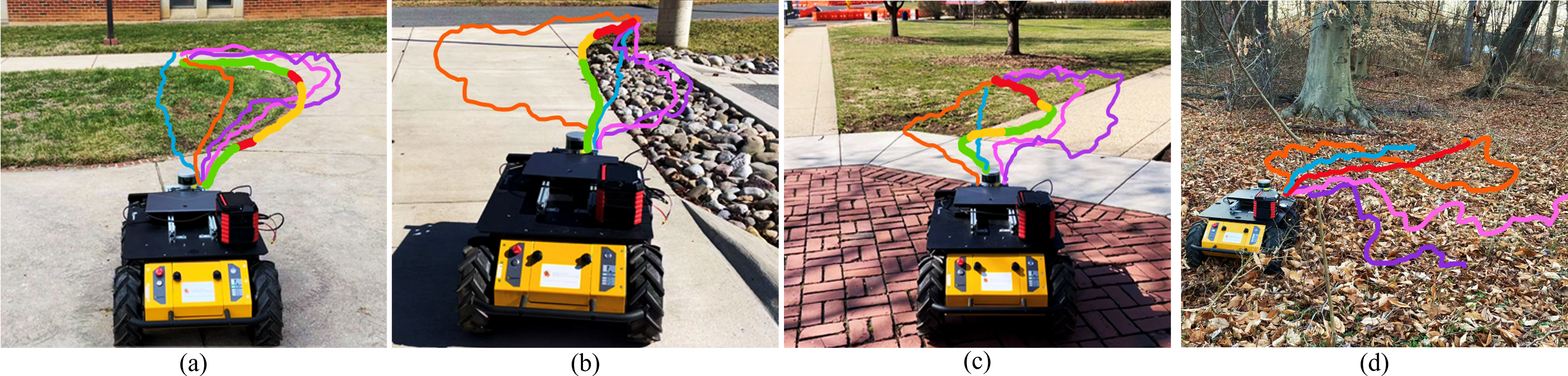}
      \caption {\small{ Robot trajectories when navigating in different unstructured terrains using our method, TerraPN (green, yellow, red corresponding to different speed levels), DWA \cite{DWA} (blue), TERP \cite{weerakoon2021terp} (orange), OCRNet-based (pink), PSPNet-based (purple) navigation schemes. (a) Scenario 1; (b) Scenario 2; (c) Scenario 3; (d) Scenario 5. We observe that TerraPN generates relatively shorter or comparable trajectories which navigate the robot on low-cost surfaces (concrete, asphalt, etc). Further, TerraPN varies the velocity appropriately when the robot encounters a relatively rough surface such as grass, rocks, dry leaves, or mud, while other methods always navigate the robot at its maximum speed on all surfaces. This leads to high vibrations and odometry errors in the robot. In (d), TerraPN moves slowly (red) throughout its progress towards its goal. TERP takes a much longer trajectory to the goal based on elevation changes. All other methods fail to reach the goal either due to incorrect segmentation (OCRNet, PSPNet) or high odometry errors (DWA).
      }}
      \label{fig:traj_comparison}
      \vspace{-10pt}
\end{figure*}

% Comparison Table
\begin{table}
\centering
\resizebox{\columnwidth}{!}{%
\begin{tabular}{|c|c|c|c|c|c|} 
\hline
\textbf{Metrics} & \textbf{Method} & \multicolumn{1}{|p{1cm}|}{\centering \textbf{Scenario} \\ \textbf{1}} & \multicolumn{1}{|p{1cm}|}{\centering \textbf{Scenario} \\ \textbf{2}} & \multicolumn{1}{|p{1cm}|}{\centering \textbf{Scenario} \\ \textbf{3}} & \multicolumn{1}{|p{1cm}|}{\centering \textbf{Scenario} \\ \textbf{4}}\\ [0.5ex] 
\hline
\multirow{5}{*}{\rotatebox[origin=c]{0}{\makecell{\textbf{Success}\\\textbf{Rate (\%)} \\ \textbf{(Higher is better)}}}} 
 & DWA \cite{DWA} & 100 & 70 & 79 & 53   \\
 & TERP \cite{weerakoon2021terp} & 100 & 77 & 78 & 61 \\
 & OCRNet \cite{ocrnet} & 94 & 72 & 73 & 58 \\
 & PSPNet \cite{pspnet} & 92 & 75 & 72 & 56 \\
 & TerraPN &  \textbf{100} &  \textbf{88} &  \textbf{85} &  \textbf{72} \\
\hline

\multirow{5}{*}{\rotatebox[origin=c]{0}{\makecell{\textbf{Norm.}\\\textbf{Traj.}\\\textbf{Length} \\ \textbf{(Close to 1 is better)}}}} 
 & DWA \cite{DWA} & 0.965 & 1.001 & 1.076  & 1.141   \\
 & TERP \cite{weerakoon2021terp} & 0.996 & 1.271 & 1.158 & 1.229 \\
 & OCRNet \cite{ocrnet} & 1.389 & 1.084 & 1.403  & 1.287 \\
 & PSPNet \cite{pspnet} & 1.241 & 1.066 & 1.524  & 1.369 \\
 & TerraPN & 1.147 & 1.034 & 1.127  & 1.261 \\
\hline

% \multirow{5}{*}{\rotatebox[origin=c]{0}{\makecell{\textbf{Goal }\\\textbf{Heading Deviation}}}}  
%  & DWA \cite{DWA} & - & - & - & -   \\
%  & TERP \cite{weerakoon2021terp} & - & - & - & - \\
%  & OCRNet & - & - & - & - \\
%  & PSPNet & - & - & - & - \\
%  & TerraPN & - & - & - & - \\
% \hline

\multirow{5}{*}{\rotatebox[origin=c]{0}{\makecell{\textbf{Vibration }\\\textbf{Cost} \\ \textbf{(lower is better)}}}}  
 & DWA \cite{DWA} & 2.334 & 1.678 & 1.518  & 3.652   \\
 & TERP \cite{weerakoon2021terp} & 1.199 & \textbf{1.279} & 1.627  & 4.156 \\
 & OCRNet \cite{ocrnet} & 0.893 & 2.115 & 1.393  & 4.378 \\
 & PSPNet \cite{pspnet} & 0.967 & 2.384 & 1.424  & 4.456 \\
 & TerraPN & \textbf{0.766} & 1.329 & \textbf{1.166}  & \textbf{2.886} \\
\hline

\multirow{5}{*}{\rotatebox[origin=c]{0}{\makecell{\textbf{Mean }\\\textbf{Velocity} \\ \textbf{(lower is better)} }}} & DWA \cite{DWA} & 0.581 & 0.564 & 0.531 & 0.542   \\
 & TERP \cite{weerakoon2021terp} & 0.544 & 0.506 & 0.525 & 0.522 \\
 & OCRNet \cite{ocrnet} & 0.561 & 0.532 & 0.529 & 0.515 \\
 & PSPNet \cite{pspnet} & 0.548 & 0.526 & 0.509 & 0.513 \\
 & TerraPN & \textbf{0.347} & \textbf{0.336} & \textbf{0.271} & \textbf{0.296} \\
\hline

\end{tabular}
}
\caption{\small{Relative performance of our method TerraPN compared to other methods on various metrics. We observe that TerraPN leads to the highest success rates in all the scenarios. Further, TerraPN results in vibration cost decrease up to 21.52\%, mean velocity reduction up to 46.76\%, and shorter or comparable trajectory lengths than the other segmentation methods in most of the scenarios. DWA  and TERP  take  shorter  trajectories in  some  cases  since  they  directly  move  towards  the  goal in  the  absence  of  obstacles  and elevation  changes, without considering surface properties. }
}
\label{tab:comparison_table}
\vspace{-10pt}
\end{table}

% Comparing Inference Time
\begin{table}[]
\centering
\resizebox{0.9\columnwidth}{!}{
\begin{tabular}{|c|c|c|}
\hline
\textbf{Method} & \textbf{Inference Time (sec)} & \textbf{Training Time}  \\
\hline
OCRNet &  0.052 & 10hrs and 5mins\\
\hline
PSPNet & 0.045 & 2hrs and 47mins\\
\hline
CGNet & \textbf{0.015} & 9hrs and 53mins \\
\hline
TerraPN-50 &  0.055 & \textbf{20-25mins}\\
\hline
TerraPN-non-uniform &  0.029 & \textbf{20-25mins} \\
\hline
\end{tabular}
}
\caption{\small{TerraPN's inference time and training time compared to existing semantic segmentation methods are executed on a laptop with Nvidia RTX2080 GPU. We observe that TerraPN's non-uniform sampling approach reduces the inference time significantly compared to PSPNet, OCRNet, and uniform sampling of $50 \times 50$ patches (TerraPN-50). Even though CGNet outperforms all the other methods in terms of inference time, it needs $\sim 10$ hours of training time to achieve satisfactory segmentation accuracy. However, TerraPN achieves cost prediction results suitable for outdoor navigation within $\sim 25$ minutes of training for 5 different surfaces.}}
\label{tab:inference-time}
\vspace{-15pt}
\end{table}

\subsubsection{Cost Map Prediction Performance} 
Fig. \ref{fig:cost-map-predict} shows the results of our cost map prediction in scenarios with various surfaces, both seen during training (tiles, grass) and unseen (soil, rocks, human obstacles). For unseen surfaces, our network typically predicts higher costs for navigation, which results in cautious, slow trajectories near such surfaces. We observe a clear differentiation of the surfaces (blue denoting low-cost surfaces) based on the robot's current velocity. Additionally, non-uniform patch sampling uses smaller patches in certain regions of interest such as the boundaries between different surfaces, and even over a pedestrian's shoes. In all other regions, larger patches are used, achieving a good balance between the fineness of the cost map and computation time.

% Add content when we compare with OCRnet and PSPnet
OCRNet and PSPNet mostly segment different surfaces as traversable (green) and moderately traversable (cyan). However, in many instances small regions of forbidden (red) and obstacle (blue) classifications are observed. This results in wandering trajectories during navigation in many scenarios. In some cases with grass and asphalt, both surfaces are considered to be equally traversable. TerraPN on the other hand, clearly prefers asphalt over grass. This reflects the discrepancy between obtaining labels from the robot-terrain interactions (in TerraPN), versus human annotations (prior segmentation methods). 

In the case with red tiles and concrete (row 3), TerraPN predicts both surfaces to have similar surface costs for the Husky robot, with high costs at the boundary between them (with soil). TerraPN accurately correlates the texture of the red tiles (previously unseen) with the tiles it has seen during training (similar to row 1 in Fig. \ref{fig:cost-map-predict}). OCRNet predicts that the red tile surface is forbidden, and concrete to be moderately traversable. PSPNet is more accurate. However, it too has incorrect forbidden segmentation in the red tiles in some instances. 

TerraPN sometimes does get confused. For example, in row 4, the different shades of the asphalt are considered to be different surfaces and slightly different costs are predicted (green versus blue). However, the small change in the predicted surface cost does not affect the robot's trajectories.

\subsubsection{Navigation Performance}
We evaluate TerraPN's navigation performance both quantitatively (Table \ref{tab:comparison_table}) and qualitatively (Fig. \ref{fig:cover-image} and \ref{fig:traj_comparison}) in our test scenarios. We do not report the metrics for scenario 5 since other methods mostly failed to reach the goal. Quantitatively, we observe that TerraPN leads to the highest success rates in all the scenarios. Other methods failed to reach the goal by either traversing surfaces where the robot got stuck such as rocks (scenario 2), or mud (scenario 3, 4). Or, in the cases of OCRnet and PSPnet, their trajectories wandered away from the goal due to the mis-classification of certain surfaces as forbidden. 
% In some cases, the high unstructuredness of the surface confuses TERP, OCRnet, and PSPnet and move the robot 

In terms of trajectory lengths, TerraPN is comparable or lower to the navigation methods based on OCRnet \cite{ocrnet} and PSPnet \cite{pspnet}. This is again due to wandering trajectories as explained before. DWA and TERP take shorter trajectories in some cases since they directly move towards the goal in the absence of obstacles and major elevation changes, without considering surface variations. However, in the cases where the normalized length is less than 1, DWA and TERP wrongly assumed that the robot had reached its goal due to odometry errors (due to wheel slip) while navigating with high velocities on high-cost surfaces. Although TerraPN's trajectories deviate to avoid high-cost surfaces, their normalized lengths are close to one. This is due to the influence of the \textit{head()} term in the objective function.

TerraPN also outperforms the other methods in terms of vibration costs in three scenarios since its trajectories were on smooth, low-cost surfaces as much as possible (see Fig. \ref{fig:cover-image} and \ref{fig:traj_comparison}) with appropriate speeds. In scenario 2, TERP outperforms TerraPN due to the presence of concrete curbs. TerraPN could not distinguish the elevation changes between the ground and the curb, and moved over its edge near the rocks (see Fig. \ref{fig:traj_comparison}b) in some trials. Since TERP uses elevation maps to compute cost maps, it takes an overly conservative trajectory to avoid the curb. OCRNet and PSPNet have higher vibration costs although they possess semantic information about the terrains due to their longer, meandering trajectories.

In terms of mean velocity, we observe that TerraPN varies the robot's speed when navigating on different surfaces in the various scenarios (see yellow and red trajectories in Figs. \ref{fig:cover-image}, \ref{fig:traj_comparison}). All other methods move towards the robot's goal close to its maximum velocity (0.6 m/s). TerraPN's overall lower mean velocity varies depending on the number and type of surfaces in the scenario (which also reflects on the vibration cost). The regions where TerraPN slowed down the robots can be seen in Figs. \ref{fig:cover-image} and \ref{fig:traj_comparison}. In the highly unstructured scenario 5, only TERP and TerraPN reached the robot's goal. TerraPN took a slow, cautious speed towards the robot's goal while TERP took a slow winding path based on the elevation changes it sensed in the terrain. DWA moved fast directly towards the goal and stopped short of it due to odometry errors caused due to the surface's poor traction. 

\subsubsection{Inference and Training Times}
Table \ref{tab:inference-time} compares TerraPN's cost prediction inference time with various semantic segmentation methods trained for outdoor environments. TerraPN's inference time when using a uniform patch sampling of $50 \times 50$ patches, is comparable to OCRNet and PSPNet. However, our non-uniform patch sampling when using patches of side lengths 50, 100, and 200 reduces the inference time by $47.27\%$ compared to uniform sampling. For comparison, we also include CGNet \cite{cgnet}, an extremely light-weight ($<500K$ parameters) segmentation network. 

However, TerraPN's cost prediction network trains in $\sim 20-25$ minutes to achieve a performance suitable for navigation compared to $\sim 3-10$ hours need to train other networks for satisfactory segmentation accuracy. Additionally, using TerraPN's cost map leads to superior navigational performance, as observed before. Additional results and analysis can be found in \cite{terrapn-arxiv}.
\section{Conclusions, Limitations and Future Work}

We present TerraPN, a novel approach that uses self-supervised learning to identify the surface characteristics in complex outdoor environments to perform autonomous robot navigation. Our method incorporates RGB images of surfaces and the robot’s velocities as inputs, and the IMU vibrations and odometry errors encountered by the robot while  traversing on a surface as labels for self-supervision. The trained self-supervised network outputs a surface navigability cost map that differentiates smooth, high-traction surfaces from bumpy, deformable surfaces. We introduce a novel navigation algorithm that accounts for the surface cost, computes cost-based acceleration limits for the robot, and computes dynamically feasible and collision-free trajectories. We validate our approach on real-world unstructured terrains and compare it with the state-of-the-art navigation techniques on various navigation metrics.

Our approach has a few limitations. TerraPN must navigate on a surface to estimate the corresponding navigability costs. This strategy cannot be applied on completely non-traversable surfaces (e.g. swampy terrains). Significant lighting changes in the environment could adversely affect the surface cost prediction. Further analysis on the number of surfaces our method can learn effectively must be conducted. TerraPN cannot distinguish subtle elevation changes in the same or similar surfaces due to the unavailability of depth and elevation inputs to the system. To this end, an intelligent multi-sensor fusion-based approach could be utilized to identify both surface and geometric properties in challenging outdoor conditions in the future.  

\bibliographystyle{IEEEtran}
\bibliography{References}
\end{document}